\documentclass[11pt]{article}
\usepackage{amsmath,amssymb,amsthm,times}
\usepackage{fullpage}
\usepackage{enumerate}
\usepackage{color}
\usepackage[colorlinks,
            linkcolor=blue,
            citecolor=blue,
            urlcolor=magenta,
            linktocpage,
            plainpages=false]{hyperref}
\usepackage{nameref}
\usepackage{natbib}
\begin{document}

\title{Distributed Learning, Communication Complexity and Privacy}

\author{Maria-Florina Balcan\thanks{School of Computer Science,
Georgia Institute of Technology.  {\tt ninamf@cc.gatech.edu}. This
work was supported by NSF grant CCF-0953192, by ONR grant
N00014-09-1-0751, and  by a Microsoft Faculty Fellowship.}
 \and
 Avrim Blum\thanks{Computer Science Department, Carnegie Mellon
University. {\tt avrim@cs.cmu.edu}.
This work was supported in part by the National Science Foundation 
under grants CCF-1116892 and IIS-1065251, and by a grant from the
United States-Israel Binational Science Foundation (BSF).}
 \and
 Shai Fine\thanks{IBM Research. {\tt shai@il.ibm.com}.}
 \and
 Yishay Mansour\thanks{School of Computer Science, Tel Aviv
University.  {\tt mansour@tau.ac.il}.  Supported in
part by The Israeli Centers of Research Excellence (I-CORE) program,
(Center  No. 4/11), by the Google Inter-university center
for Electronic Markets and Auctions, by a grant from the Israel
Science Foundation, by a grant from United States-Israel Binational
Science Foundation (BSF), and by a grant from the Israeli Ministry
of Science (MoS).}
 }

\newcommand{\ignore}[1]{}
\newcommand{\HH}{{\cal H}}
\newcommand{\E}{{\bf E}}
\newcommand{\nplayers}{k}
\newcommand{\Players}{K}
\newcommand{\opt}{{\sf opt}}

\newtheorem{theorem}{Theorem}
\newtheorem{lemma}[theorem]{Lemma}
\newtheorem{claim}{Claim}[theorem]
\newtheorem{defn}{Definition}
\newtheorem{obs}[theorem]{Observation}

\date{}

\maketitle

\begin{abstract}
We consider the problem of PAC-learning from distributed data and
analyze fundamental communication complexity questions involved.
We provide general upper and lower bounds on the amount of
communication needed to learn well, showing that in
addition to VC-dimension and covering number, quantities such as the
teaching-dimension and mistake-bound of a class play an important
role.  We also present tight results for a number of common
concept classes including conjunctions, parity functions, and decision
lists.  For linear separators, we show that for non-concentrated
distributions, we can use a version of the Perceptron algorithm to
learn with much less communication than the number of updates given by
the usual margin bound.  We also show how boosting can be performed in
a generic manner in the distributed setting to achieve communication
with only logarithmic dependence on $1/\epsilon$ for any concept
class, and demonstrate how recent work on agnostic learning from
class-conditional queries can be used to achieve low communication in
agnostic settings as well.  We additionally present an analysis of
privacy, considering both differential privacy and a notion of
distributional privacy that is especially appealing in this context.
\end{abstract}


\section{Introduction}

Suppose you have two databases: one with the positive examples and
another with the negative examples.  How much communication between
them is needed to learn a good hypothesis?  In this paper we
consider this question and its generalizations, as well
as related issues such as privacy.  Broadly, we consider a framework
where information is distributed between several locations, and our
goal is to learn a low-error hypothesis with respect to the overall
distribution of data using as little communication, and as few
rounds of communication, as possible.  Motivating examples include:

\begin{enumerate}
\item Suppose $k$ research groups around the world have collected
large scientific datasets, such as genomic sequence data or
sky survey data, and we wish to perform learning over the
union of all these different datasets without too much communication.

\item Suppose we are a sociologist and want to understand what
distinguishes the clientele of two retailers (Macy's vs Walmart).  Each
retailer has a large database of its own customers and we want to
learn a classification rule that distinguishes them.  This is an
instance of the case of each database corresponding to a different label.
It also brings up natural privacy issues.

\item
Suppose $k$ hospitals with different distributions of patients want to
learn a classifier to identify a common misdiagnosis.
Here, in addition to the goal of achieving high accuracy, low
communication, and privacy for patients, the hospitals may want to
protect their own privacy in some formal way as well.
\end{enumerate}

We note that we are
interested in learning a single hypothesis $h$ that performs well
overall, rather than separate hypotheses $h_i$ for each database.  For
instance, in the case that one database has all the positive examples
and another has all the negatives, the latter problem becomes trivial.
More generally, we are interested in understanding the fundamental
communication complexity questions involved in distributed learning,
a topic that is becoming increasingly relevant to modern learning problems.
These issues, moreover, appear to be quite interesting even for the
case of $k=2$ entities.

\subsection{Our Contributions}

We consider and analyze fundamental communication
questions in PAC-learning from distributed data, providing
general upper and lower bounds on the amount of
communication needed to learn a given class, as well as
broadly-applicable techniques for achieving communication-efficient
learning.  We also analyze a number of important specific classes,
giving efficient learning algorithms with especially good
communication performance, as well as in some cases counterintuitive
distinctions between proper and non-proper learning.

Our general upper and lower bounds show that in addition to
VC-dimension and covering number, quantities such as the
teaching-dimension and mistake-bound of a class play an important role
in determining communication requirements.  We also show how boosting
can be performed in a communication-efficient manner, achieving
communication depending only logarithmically on $1/\epsilon$ for any
class, along with tradeoffs between total communication and number of
communication rounds.  Further we show that, ignoring computation,
agnostic learning can be performed to error $O(\opt(\HH))+\epsilon$ with
logarithmic dependence on $1/\epsilon$, by adapting results of
\cite{BH12}.

In terms of
specific classes, we present several tight bounds including
a $\Theta(d \log d)$ bound on the communication in
bits needed for learning the class of decision lists over $\{0,1\}^d$.
For learning linear separators, we show that for non-concentrated
distributions, we can use a version of the Perceptron algorithm to
learn using only $O(\sqrt{d
\log(d/\epsilon)}/\epsilon^2)$ rounds of communication, each round
sending only a single hypothesis vector, much less than the
$O(d/\epsilon^2)$ total number of updates performed by the Perceptron
algorithm.  For parity functions, we give a rather surprising result.
For the case of two entities, while proper learning has an
$\Omega(d^2)$ lower bound based on classic results in communication
complexity, we show that {\em non-proper} learning can be done
efficiently using only $O(d)$ bits of communication.  This is a
by-product of  a general result regarding concepts
learnable in the reliable-useful framework of \cite{RS88}.
For a table of results, see Appendix \ref{app:table}.

We additionally present an analysis of communication-efficient
privacy-preserving learning
algorithms, considering both {\em differential privacy} and a notion
of {\em distributional privacy} that is especially appealing in this context.
We show that in many cases we can achieve privacy without incurring
any additional communication penalty.

More broadly, in this work we propose and study communication as a fundamental
resource for PAC-learning in addition to the usual measures of time
and samples.  We remark that all our algorithms for specific classes
address communication while maintaining efficiency along the other two axes.

\subsection{Related Work}

Related work in computational learning theory
has mainly focused on the topic of learning and parallel computation.
\cite{Bshouty97} shows that many simple classes that can be PAC
learned can not be efficiently learned in parallel with a polynomial
number of processors.
\cite{LongS11} show a parallel algorithm for large margin
classifiers running in time $O(1/\gamma)$ compared to more naive
implementations costing of $\Omega(1/\gamma^2)$, where $\gamma$ is
the margin. They also show an impossibility result regarding
boosting, namely that the ability to call the weak learner oracle
multiple times in parallel within a single boosting stage does not
reduce the overall number of successive stages of boosting that are
required.
\cite{CollinsSS02} give an online algorithm that uses a
parallel-update method for the logistic loss, and \cite{ZWSL10} give a
detailed analysis of a parallel stochastic gradient descent in which
each machine processes a random subset of the overall data, combining
hypotheses at the very end.
All of the above results are mainly interested in reducing the time
required to perform learning when data can be randomly or
algorithmically partitioned among processors; in contrast, our focus
is on a setting in which we begin with data arbitrarily partitioned
among the entities.
\cite{DGSX11} consider distributed online prediction with arbitrary
partitioning of data streams, achieving strong regret bounds; however,
in their setting the goal of
entities is to perform well on their {\em own} sequence of data.

In very recent independent work, \cite{DPSV12} examine a setting much
like that considered here, in which parties each have an arbitrary
partition of an overall dataset, and the goal is to achieve low error
over the entire distribution.  They present comunication-efficient
learning algorithms for axis-parallel boxes as well as for learning
linear separators in $R^2$.  \cite{DPSV12b}, also independently of our
work, extend this to the case of linear separators in $R^d$, achieving
bounds similar to those obtained via our distributed boosting results.
Additionally, they consider a range of distributed optimization problems,
give connections to streaming algorithms, and present a number of
experimental results.  Their work overall is largely complementary to
ours.

\section{Model and Objectives}

Our model can be viewed as a distributed version of the PAC model.  We
have $\nplayers$ entities (also called ``players'') denoted by
$\Players$ and an instance space $X$.  For each entity $i\in \Players$
there is a distribution $D_i$ over $X$ that entity $i$ can sample
from.  These samples are labeled by an unknown target function $f$.
Our goal is to find a hypothesis $h$ which approximates $f$
well on the joint mixture $D(x) =
\frac{1}{\nplayers}\sum_{i=1}^\nplayers D_i(x)$.  In the realizable
case, we are given a concept class $\HH$ such that $f \in \HH$; in
the agnostic case, our goal is to perform nearly as well as the best
$h' \in \HH$.

In order to achieve our goal of approximating $f$ well with respect to
$D$, entities can communicate with each other, for example by sending
examples or hypotheses.  At the end of the process, each entity should
have a hypothesis of low error over $D$.  In the {\em center} version
of the model there is also a center, with initially no data of its
own, mediating all the interactions.  In this case the goal is for the
center to obtain a low error hypothesis $h$.  In the {\em no-center}
version, the players simply communicate directly.  In most cases, the
two models are essentially equivalent; however (as seen in Section
\ref{sec:parity}), the case of parity functions forms a notable
exception. We assume the $D_i$ are not known to the center or to any
entity $j\neq i$ (in fact, $D_i$ is not known explicitly even by
entity $i$, and can be approximated only via sampling).
Finally, let $d$ denote the VC dimension of $\HH$, and $\epsilon$
denote our target error rate in the realizable case, or our target gap
with respect to opt$(\HH)$ in the agnostic case.\footnote{We
will suppress dependence on the confidence parameter $\delta$ except
in cases where it behaves in a nontrivial manner.}  We will typically
think of $\nplayers$ as much smaller than $d$.

Remark: We are assuming all players have the
same weight, but all results extend to players
with different given weights.
We also remark that except for our generic results, all our algorithms
for specific classes will be computationally efficient (see Appendix
\ref{app:table}).


\subsection*{Communication Complexity}

Our main focus is on learning methods that
minimize the communication needed in order
to learn well.  There are two critical parameters, the {\em total
communication} (either in terms of bits transmitted or examples or
hypotheses transmitted ) and {\em latency} (number of rounds
required).  Also, in comparison to the baseline algorithm of having
each database send all (or a random sample of) its data to a center,
we will be looking both at methods that improve over the dependence
on $\epsilon$ and that improve over the dependence on $d$ in terms
of the amount of communication needed (and in some cases we will be
able to improve in both parameters).  In both cases, we will be
interested in the tradeoffs between total communication and the
number of communication rounds. The interested reader is referred
to \cite{NK} for an excellent exposition of communication
complexity.

When defining the exact communication model, it is important to
distinguish whether entities can learn information from {\em not}
receiving any data.
For the most part we assume an {\em asynchronous} communication
model, where the entities can not deduce any information when they
do not receive the data (and there is no assumption about the delay
of a message). In a few places we use a much stronger model of {\em
lock-synchronous} communication, where the communication is in time
slots (so you can deduce that no one sent a message in a certain
time slot) and if multiple entities try to transmit at the same time
only one succeeds. Note that if we have an algorithm with $T$ time
steps and $C$ communication bits in the lock-synchronous model,
using an exponential back-off mechanism \citep{HS} and a
synchronizer \citep{peleg}, we can convert it to an asynchronous
communication with $O(T\log k)$ rounds and $O((T+C)\log k)$
communication bits.


\subsection*{Privacy}
\label{sec:introprivacy}
In addition to minimizing communication, it is also natural in this
setting to consider issues of privacy, which we examine
in Section \ref{sec:privacy}.  In particular, we will consider privacy
of three forms: {\em differential privacy for the examples} (the
standard form of privacy considered in the literature), {\em
differential privacy for the databases} (viewing each entity as an
individual deserving of privacy, which requires $\nplayers$ to be
large for any interesting statements), and {\em distributional privacy
for the databases} (a weaker form of privacy that we can achieve even
for small values of $\nplayers$).  See \cite{Dwork08} for an excellent
survey of differential privacy.

\section{Baseline approaches and lower bounds}
\label{sec:baseline}

We now describe two baseline methods for
distributed learning as well as present general lower bounds.

\smallskip
\noindent
{\em Supervised Learning:} The simplest baseline approach is to just
have each database send a random sample of size
$O(\frac{1}{k}(\frac{d}{\epsilon}\log\frac{1}{\epsilon}))$ to the
center, which then performs the learning.  This implies
we have a total communication cost of
$O(\frac{d}{\epsilon}\log\frac{1}{\epsilon})$ in terms of number
of examples transmitted.   Note that while the sample received by
the center is not precisely drawn from $D$ (in particular, it
contains the same number of points from each $D_i$), the
standard double-sample VC-dimension argument still applies, and so
with high probability all consistent $h \in \HH$ have low error.
Similarly, for the agnostic case it suffices to use a total of
$O(\frac{d}{\epsilon^2}\log\frac{1}{\epsilon})$ examples.  In both
cases, there is just one round of communication.

\smallskip
\noindent {\em EQ/online algorithms:} A second baseline method is to
run an Equivalence Query or online Mistake-Bound algorithm at the
center. This method is simpler to describe in the
lock-synchronization model. In each round the center broadcasts its
current hypothesis. If any of the entities has a counter-example, it
sends the counter-example to the center. If not, then we are done.
The total amount of communication measured in terms of examples and
hypotheses transmitted is at most the mistake bound $M$ of the
algorithm for learning $\HH$; in fact, by having each entity
run a shadow copy of the algorithm, one needs only to transmit the
examples and not the hypotheses.  Note that in comparison to the
previous baseline, there is now no dependence on $\epsilon$ in terms
of communication needed; however, the number of rounds may now be as
large as the mistake bound $M$ for the class $\HH$.
Summarizing,

\begin{theorem}
\label{thm:baseline}
Any class $\HH$ can be learned to error $\epsilon$ in the realizable
case using 1 round and $O(\frac{d}{\epsilon}\log\frac{1}{\epsilon})$
total examples communicated, or $M$
rounds and $M$ total examples communicated, where
$M$ is the optimal mistake bound for $\HH$.
In the agnostic case, we can learn to error opt$(\HH) + \epsilon$
using  1 round and $O(\frac{d}{\epsilon^2}\log\frac{1}{\epsilon})$
total examples communicated.
\end{theorem}

Another baseline approach is for each player to describe an
approximation to the joint distribution induced by $D_i$ and $f$ to
the center, in cases where that can be done efficiently.  See Appendix
\ref{app:interval} for an example.

We now present a general lower bound on
communication complexity for learning a class $\HH$.  Let
$N_{\epsilon,D}(\HH)$ denote the size of the minimum $\epsilon$-cover
of $\HH$ with respect to $D$, and let $N_\epsilon(\HH) = \sup_D
N_{\epsilon,D}(\HH)$.  Let $d_{T}(\HH)$ denote the {\em teaching
dimension} of class $\HH$.\footnote{$d_T(\HH)$ is
defined as $\max_{f \in \HH} d_T(f)$ where $d_T(f)$ is the smallest
number of examples needed to uniquely identify $f$ within $\HH$
\citep{GoldmanKe91}.}

\begin{theorem}
\label{thm:lower}
Any class $\HH$ requires $\Omega(\log N_{2\epsilon}(\HH))$ bits of
communication to learn to error $\epsilon$.  This implies
$\Omega(d)$ bits are required to learn to error $\epsilon \leq 1/8$.
For proper learning, $\Omega(\log |\HH|)$ bits are required to learn to error
$\epsilon < \frac{1}{2d_T(\HH)}$.   These hold even for $k=2$.
\end{theorem}

\begin{proof}
Consider a distribution $D_1$ such that $N = N_{2\epsilon,D_1}(\HH)$ is
maximized.  Let $D_2$ be concentrated on a single (arbitrary) point $x$.
In order for player 2 to produce a hypothesis $h$ of error at most
$\epsilon$ over $D$, $h$ must have error at most $2\epsilon$ over
$D_1$.  If player 2 receives fewer than $\log_2(N_{2\epsilon}(\HH))-1$
bits from player 1, then (considering also the two possible labels of
$x$) there are less than $N_{2\epsilon}(\HH)$ possible hypotheses
player 2 can output.  Thus, there must be some $f \in \HH$ that has
distance greater than $2\epsilon$ from all such hypotheses with
respect to $D_1$, and so player 2 cannot learn that function.
The $\Omega(d)$ lower bound follows from applying the above argument
to the uniform distribution over $d$ points shattered by $\HH$.

For the $\Omega(\log |\HH|)$ lower bound, again let $D_2$ be concentrated
on a single (arbitrary) point.  If player 2 receives fewer than
$\frac{1}{2}\log |\HH|$ bits then there must be some $h^* \in \HH$ it
cannot output.  Consider $f=h^*$ and let $D_1$ be uniform over
$d_T(\HH)$ points uniquely defining $f$ within $\HH$.  Since player 2
is a proper learner, it must
therefore have error greater than $2\epsilon$ over $D_1$, implying
error greater than $\epsilon$ over $D$.
\end{proof}

Note that there is a significant gap between the above upper
and lower bounds.  For instance, if data lies in $\{0,1\}^d$, then in
terms of $d$ the
upper bound in {\em bits} is $O(d^2)$ but the lower bound is
$\Omega(d)$ (or in {\em examples}, the upper bound is $O(d)$
but the lower bound is $\Omega(1)$).
In the following sections, we describe our algorithmic results for
improving upon the above baseline methods, as well as stronger
communication lower bounds for certain classes.  We also show how
boosting can be used to generically get only a logarithmic dependence
of communication on $1/\epsilon$ for any class, using a logarithmic
number of rounds.

\section{Intersection-closed classes and version-space algorithms}
\label{sec:iclosed}
One simple case where one can perform substantially better than the
baseline methods is that of intersection-closed (or union-closed)
classes $\HH$, where the functions in $\HH$ can themselves be
compactly described.  For example, the class of conjunctions and the
class of intervals on the real line are both intersection-closed.
For such classes we have the following.
\begin{theorem}
\label{thm:iclosed}
If $\HH$ is intersection-closed, then $\HH$ can be learned using one
round and $k$ hypotheses of total communication.
\end{theorem}
\begin{proof}
Each entity $i$ draws a sample of size
$O(\frac{1}{\epsilon}(d\log(\frac{1}{\epsilon}) + \log(k/\delta)))$
and computes the smallest
hypothesis $h_i \in \HH$ consistent with its sample, sending $h_i$
to the center. The center then computes the smallest
hypothesis $h$ such that $h \supseteq h_i$ for all $i$.  With
probability at least $1-\delta$, $h$ has error at most $\epsilon$ on
each $D_i$ and therefore error at most $\epsilon$ on $D$ overall.
\end{proof}
\noindent
{\em Example (conjunctions over $\{0,1\}^d$):} In this case, the above
procedure corresponds to each player sending the bitwise-and of all its
positive examples to the center.  The center then computes the
bitwise-and of the results.  The total communication in bits is $O(dk)$.
Notice this may be substantially smaller than the $O(d^2)$ bits used
by the baseline methods.

\medskip
\noindent
{\em Example (boxes in $d$-Dimensions):} In this case, each player
can send its smallest consistent hypothesis using $2d$
values.  The center examines the minimum and maximum in each coordinate to
compute the minimal $h \supseteq h_i$ for all $i$.
Total communication is $O(dk)$ values.

\medskip
\noindent
In Appendix~\ref{app:ver} we discuss related algorithms based on version
spaces.

\section{Reliable-useful learning, parity, and lower bounds}
\label{sec:parity}

A classic lower bound in communication complexity states that if two
entities each have a set of linear equalities over $n$ variables,
then $\Omega(n^2)$ bits of communication are needed to determine a
feasible solution, based on \cite{JaJaK84}.  This in turn implies
that for {\em proper} learning of parity functions, $\Omega(n^2)$
bits of communication are required even in the case $k=2$, matching
the baseline upper bound given via Equivalence Query algorithms.

Interestingly, however, if one drops the requirement that learning be
proper, then for $k=2$, parity functions {\em can} be learned using
only $O(n)$ bits of communication.  Moreover, the algorithm is efficient.
This is in fact a special case of the following result for classes
that are learnable in the {\em reliable-useful} learning model of
\cite{RS88}.

\begin{defn}\citep{RS88} An algorithm {\em reliably and usefully}
learns a class $\HH$ if
given $poly(n,1/\epsilon,1/\delta)$ time and samples, it produces a
hypothesis $h$ that on any given example outputs either a correct
prediction or the statement ``I don't know''; moreover, with
probability at least $1-\delta$ the probability mass of examples for
which it answers ``I don't know'' is at most $\epsilon$.
\end{defn}

\begin{theorem}
Suppose $\HH$ is properly PAC-learnable and is learnable (not
necessarily properly) in the reliable-useful model.
Then for $k=2$, $\HH$ can be learned in one round with 2 hypotheses of
total communication (or $2b$ bits of communication if each $h \in \HH$
can be described in $b=O(\log|\HH|)$ bits).
\end{theorem}
\begin{proof}
The algorithm is as follows.  First, each player $i$ properly
PAC-learns $f$ under $D_i$ to error $\epsilon$, creating hypothesis
$h_i \in \HH$.  It also
learns $f$ reliably-usefully to create hypothesis $g_i$ having
don't-know probability mass at most $\epsilon$ under $D_i$.
Next, each player $i$ sends $h_i$ to the other player (but not $g_i$,
because $g_i$ may take too many bits to communicate since it is not
guaranteed to belong to $\HH$).
Finally, each player $i$ produces the overall hypothesis ``If my own
$g_i$ makes a prediction, then use it; else use the hypothesis
$h_{3-i}$ that I received from the other player''.  Note that each
player $i$'s final hypothesis has error at most $\epsilon$ under both
$D_i$ (because of $g_i$) and $D_{3-i}$ (because $h_{3-i}$ has error at
most $\epsilon$ under $D_{3-i}$ and $g_i$ never makes a mistake) and
therefore has error at most $\epsilon$ under $D$.
\end{proof}
{\em Example (parity functions):} Parity functions are properly PAC
learnable (by an arbitrary consistent solution to the
linear equations defined by the sample).  They are also
learnable in the reliable-useful model by a (non-proper) algorithm
that behaves as follows: if the given test example $x$ lies in the
span of the training data, then write $x$ as a sum of training
examples and predict the corresponding sum of labels.  Else output
``I don't know''.  Therefore, for $k=2$, parity functions are
learnable with only $O(n)$ bits of communication.

\smallskip
Interestingly, the above result does {\em not} apply to the case in
which there is a center that must also learn a good hypothesis.  The
reason is that the output of the reliable-useful learning procedure
might have large bit-complexity, for example, in the case of parity
it has a complexity of $\Omega(n^2)$. A similar problem arises when
there are more than two entities.\footnote{It is interesting to note
that if we allow communication in the classification phase (and not only
during learning) then the center can simply send each test example to
all entities, and any entity that classifies it has to be correct.}

However, we {\em can} extend the result to the case of a center
if the overall distribution $D$ over {\em unlabeled} data is known to the
players.  In particular, after running the above protocol to error
$\epsilon/d$, each player can then draw $O(d/\epsilon)$ fresh unlabeled points
from $D$, label them using its learned hypothesis, and then perform
proper learning over this data to produce a new hypothesis $h' \in
\HH$ to send to the center.

\section{Decision Lists}

We now consider the class $\HH$ of decision lists over $d$
attributes. The best mistake-bound known for this class is $O(d^2)$,
and its VC-dimension is $O(d)$.  Therefore, the baseline algorithms
give a total communication complexity, in bits,
of $\tilde{O}(d^2/\epsilon)$ for batch learning and $O(d^3)$ for the
mistake-bound algorithm.\footnote{One simple observation is the
communication complexity of the mistake-bound algorithm can be
reduced to $O(d^2 \log d)$ by having each player, in the event of a
mistake, send only the identity of the offending rule rather than
the entire example; this requires only $O(\log d)$ bits per mistake.
However we will be able to beat this bound substantially.}
Here, we present an improved algorithm, requiring a total
communication complexity of only $O(dk\log d)$ bits.  This is a
substantial savings over both baseline algorithms, especially when
$k$ is small. Note that for constant $k$ and for $\epsilon =
o(1/d)$, this bound matches the proper-learning $\Omega(d \log d)$
lower bound of Theorem \ref{thm:lower}.

\begin{theorem}
\label{thm:declist}
The class of decision lists can be efficiently learned with a total of
at most $O(dk\log d)$ bits of communication and a number of rounds
bounded by the number of alternations in the target decision list
$f$.
\end{theorem}
\begin{proof}
The algorithm operates as follows.
\begin{enumerate}
\item First, each player $i$ draws a sample $S_i$ of size
$O(\frac{1}{\epsilon}(d\log(\frac{1}{\epsilon}) + \log(k/\delta)))$,
which is sufficient so that consistency with $S_i$ is sufficient for
achieving low error over $D_i$.
\item Next, each player $i$ computes the set $T_i$ of all
triplets $(j,b_j,c_j)$ such that the rule ``if $x_j=b_j$ then
$c_j$'' is consistent with all examples in $S_i$.  (For convenience,
use $j=0$ to denote the rule ``else $c_j$''.) Each player $i$ then
sends its set $T_i$ to the center.
\item The center now computes the intersection of all sets $T_i$ received
and broadcasts the result $T = \cap_i T_i$ to all players, i.e., 
the collection of triplets consistent with every $S_i$.
\item Each player $i$ removes from $S_i$ all examples satisfied by
$T$.
\item Finally, we repeat steps 2,3,4 but in Step 2 each player only
sending to the center any {\em new} rules that have become
consistent since the previous rounds (the center will add them into
$T_i$---note that there is never a need to delete any rule from
$T_i$); similarly in Step 3 the center only sends {\em new} rules
that have entered the intersection $T$. The
process ends once an ``else $c_j$'' rule has entered $T$.  The final
hypothesis is the decision list consisting of the rules broadcast by
the center, in the order they were broadcast.
\end{enumerate}
To analyze the above procedure, note first that since each player
announces any given triplet at most once, and any triplet can be
described using $O(\log d)$ bits, the
total communication in bits per player is at most $O(d \log d)$, for
a total of $O(dk\log d)$ overall.  Next, note that the topmost rule
in $f$ {\em will} be consistent with each $S_i$, and indeed so will
all rules appearing before the first alternation in $f$.  Therefore,
these will be present in each $T_i$ and thus contained in $T$.  Thus,
each player will remove all examples exiting through any
such rule.  By induction, after $k$ rounds of the protocol, all
players will have removed all examples in their datasets that exit
in one of the top $k$ alternations of $f$, and therefore in the next
round all rules in the $k+1$st alternation of $f$ that have not been
broadcast already will be output by the center.  This implies the
number of rounds will be bounded by the number of alternations of
$f$.  Finally, note that the hypothesis produced will
by design be consistent with each $S_i$ since a new rule is added to
$T$ only when it is consistent with every $S_i$.
\end{proof}
\section{Linear Separators}
We now consider the case of learning homogeneous linear separators
in $R^d$.
For this problem, we will for convenience discuss communication in
terms of the number of vectors transmitted, rather than bits.
However, for data of margin $\gamma$, all vectors transmitted can be
given using $O(d \log 1/\gamma)$ bits each.

One simple case is when $D$ is a radially symmetric distribution
such as the symmetric Gaussian distribution centered at the origin,
or the uniform distribution on the sphere.  In that case, it is
known that $\E_{x \sim D}[\ell(x)x/||x||]$,
is a vector exactly in the direction of the target vector, where
$\ell(x)$ is the label of $x$. Moreover, an average over
$O(d/\epsilon^2)$ samples is sufficient to produce an estimate of
error at most $\epsilon$ with high probability \citep{Servedio02}.
Thus, so long as each player draws a sufficiently large sample $S_i$,
we can learn to
any desired error $\epsilon$ with a total communication of only $k$
examples: each database simply computes an average over its own data
and sends it to the center, which combines the
results.\ignore{\footnote{This is a slightly different experiment than just
sampling from $D$, because it corresponds to averaging a sample in
which exactly the same number of points are drawn from each $D_i$,
but the results of \cite{Servedio02} still apply.}}

The above result, however, requires very precise conditions on the overall
distribution.  In the following we consider several more
general scenarios: learning
a large-margin separator when data is ``well-spread'', learning over
non-concentrated distributions, and learning linear separators
without any additional assumptions.
\subsection{Learning  large-margin separators when
data is well-spread}
\label{sec:wellspread}
We say that data is {\em $\alpha$-well-spread} if for all datapoints
$x_i$ and $x_j$ we have  $\frac{|x_i \cdot x_j |}{||x_i|| ||x_j||} < \alpha$.  In the
following we show that if data is indeed $\alpha$-well-spread for a
small value of $\alpha$, then the
Perceptron algorithm can be used to learn with substantially less
communication than that given by just using its mistake-bound directly
as in Theorem \ref{thm:baseline}.

\begin{theorem}\label{lem:orthogonal}
Suppose that data is $\alpha$-well-spread and furthermore that
all points have margin at least $\gamma$ with the target
$w^*$.  Then we can find a consistent hypothesis with a version of
the Perceptron algorithm using at most $O(k(1+\alpha/\gamma^2))$ rounds
of communication, each round communicating a single hypothesis.
\end{theorem}

\begin{proof}
We will run the algorithm in meta-rounds. Each meta-round will
involve a round robin communication between the players $1, \ldots, k$.
Starting from initial hypothesis $w_0 = \vec{0}$, each player $i$
will in turn run the Perceptron algorithm on its data until it finds
a consistent hypothesis $w_{t,i}$ that moreover satisfies $|w_t\cdot
x_i| >1$ for all of its examples $x_i$.  It then sends the
hypothesis $w_{t,i}$ produced to player $i+1$ along with the number of
updates it performed, who continues this algorithm on its own data,
starting from the most recent hypothesis $w_{t,i}$. When player $k$
sends $w_{t,k}$ to player $1$, we start meta-round $t+1$. At the
start of meta-round $t+1$, player $1$ counts the number of updates
made in the previous meta-round, and if it is less than $1/\alpha$
we stop and output the current hypothesis.

It is known that this ``Margin Perceptron'' algorithm makes at most
$3/\gamma^2$ updates in total.\footnote{Because after update
$\tau$ we get $||w_{\tau+1}||^2 \leq ||w_\tau||^2 + 2\ell(x_i)(w_\tau
\cdot x_i) + 1 \leq ||w_\tau||^2 + 3$.}
%
Note that if in a meta-round all the players make less than $1/\alpha$
updates in total, then we know the hypothesis will still be
consistent with all players' data.  That is because each update
can decrease the inner product of the hypothesis with some $x_i$ of
another player by at most $\alpha$.  So, if less than $1/\alpha$
updates occur, it implies that every player's examples are still
classified correctly. This implies that the total number of
communication meta-rounds until a consistent hypothesis is produced
will be at most $1+ 3\alpha/\gamma^2$.  In particular, this follows
because the total number of updates is at most $3/\gamma^2$, and
each round, except the last, makes at least $1/\alpha$ updates.
\end{proof}

\subsection{Learning linear separators over non-concentrated distributions}
We now use the analysis of Section \ref{sec:wellspread} to achieve
good communication bounds for learning linear separators over
{\em non-concentrated} distributions.
Specifically, we say a
distribution over the $d$-dimensional unit sphere is
non-concentrated if for some constant $c$, the probability density
on any point $x$ is at most $c$ times greater than that of the
uniform distribution over the sphere.   The key idea is
that in a non-concentrated distribution, nearly all pairs of points will
be close to orthogonal, and most points will have reasonable margin
with respect to the target.

\begin{theorem}
\label{thm:perceptron} For any non-concentrated distribution $D$
over $R^d$ we can learn to error $O(\epsilon)$ using only
$O(k^2\sqrt{d \log(dk/\epsilon)}/\epsilon^2)$  rounds of communication,
each round communicating a single hypothesis vector.
\end{theorem}

\begin{proof}
Note that for any non-concentrated distribution $D$, the probability
that two random examples $x,x'$ from $D$ satisfy $|x \cdot x'| >
t/\sqrt{d}$ is $e^{-O(t^2)}$.  This implies that in a
polynomial-size sample (polynomial in $d$ and $1/\epsilon$), with
high probability, any two examples $x_i,x_j$ in the sample satisfy
$|x_i \cdot x_j| \leq \sqrt{c'\log(d/\epsilon)/n}$ for some constant
$c'$.  Additionally, for any such distribution $D$ there exists
another constant $c''$ such that for any $\epsilon>0$, there is at
most $\epsilon$ probability mass of $D$ that lies within margin
$\gamma_\epsilon = c''\epsilon/\sqrt{d}$ of the target.

These together imply that using the proof idea of Theorem
\ref{lem:orthogonal}, we can learn to error $O(\epsilon)$ using only
$O(k^2\sqrt{d\log(dk/\epsilon)}/\epsilon^2)$ communication rounds.
Specifically, each player acts as follows.  If the hypothesis $w$
given to it has error at most $\epsilon$ on its
own data, then it makes no updates and just passes $w$ along.
Otherwise, it makes updates
using the margin-perceptron algorithm by choosing random examples
$x$ from its own distribution $D_i$ satisfying $\ell(x)(w \cdot x) <
1$ until the fraction of examples $x$ under $D_i$ for which
$\ell(x)(w \cdot x) < 1$ is at most $\epsilon$, sending the final
hypothesis produced to the next player.  Since before each update,
the probability mass under $D_i$ of $\{x : \ell(x)(w \cdot x) < 1\}$
is at least $\epsilon$, the probability mass of this region under
$D$ is at least $\epsilon/(2k)$. This in turn means there is at least a
$1/2$ probability that the example used for updating has margin at
least $\frac{1}{2}\gamma_{\epsilon/(2k)} =
\Omega(\epsilon/(k\sqrt{d}))$ with respect to the target.  Thus, the
total number of updates made over the entire algorithm will be only
$O(dk^2/\epsilon^2)$.  Since the process will halt if all
players make fewer than $1/\alpha$ updates in a meta-round, for
$\alpha=\sqrt{c'\log(2dk/\epsilon)/n}$, this implies the total number
of communication meta-rounds is $O(k^2\sqrt{d\log(d/\epsilon)}/\epsilon^2)$.
\end{proof}

Note that in Section \ref{sec:boosting} we show how boosting can be
implemented communication-efficiently so that any class
learnable to constant error rate from a sample of size $O(d)$ can be
learned to error $\epsilon$ with
total communication of only $O(d \log 1/\epsilon)$ examples (plus a
small number of additional bits).
However, as usual with boosting, this requires a
distribution-independent weak learner.  The ``$1/\epsilon^2$'' term in
the bound of Theorem \ref{thm:perceptron} comes from the margin
that is satisfied by a $1-\epsilon$ fraction of points under a
non-concentrated distribution, and so the results of Section
\ref{sec:boosting} do not eliminate it.

\subsection{Learning linear separators without any additional assumptions}
%
If we are willing to have a bound that depends on the dimension $d$,
then we can run a mistake-bound algorithm for learning linear
separators, using Theorem \ref{thm:baseline}.
Specifically, we can use a mistake-bound algorithm based on reducing
the volume of the version space of consistent
hypotheses (which is a polyhedra). The initial volume is $1$ and the
final volume is $\gamma^d$, where $\gamma$ is the margin of the
sample. In every round,
each player checks if it has an example that reduces the volume by
half (volume of hypotheses consistent with all examples broadcast so
far).  If it does, it sends it (we are using here the
lock-synchronization model). If no player has such an example, then we
are done.
The hypothesis we have is for each $x$ to predict with the majority of
the consistent hypotheses. This gives a total of $O(d\log 1/\gamma)$ examples
communicated.  In terms of bits, each example has $d$ dimensions, and
we can encode each dimension with $O(\log 1/\gamma)$ bits, thus the
total number of bits communicated is $O(d^2\log^2 1/\gamma)$.
Alternatively, we can replace the $\log 1/\gamma$ term with a $\log
1/\epsilon$ term by using a PAC-learning algorithm to learn to
constant error rate, and then applying the boosting results
of Theorem \ref{thm:boosting} in Section \ref{sec:boosting} below.

\medskip

It is natural to ask whether running the Perceptron algorithm in a
round-robin fashion could be used to improve the generic
$O(1/\gamma^2)$ communication bound given by the baseline results of
Theorem \ref{thm:baseline}, for general distributions of margin $\gamma$.
However, in Appendix \ref{app:per_margin} we present
an example where the Perceptron algorithm indeed requires
$\Omega(1/\gamma^2)$ rounds.

\begin{theorem}
\label{claim:per_margin}
There are inputs for $k=2$ with margin $\gamma$ such that the
Perceptron algorithm takes $\Omega(1/\gamma^2)$ rounds.
\end{theorem}

\section{Boosting for Logarithmic Dependence on $1/\epsilon$}
\label{sec:boosting}
We now consider the general question of dependence of communication
on $1/\epsilon$, showing how boosting can be used to achieve
$O(\log 1/\epsilon)$ total communication in $O(\log 1/\epsilon)$
rounds for any concept class, and
more generally a tradeoff between communication and rounds.

Boosting algorithms provide a mechanism to produce an $\epsilon$-error
hypothesis given access only to a weak learning oracle, which on any
distribution finds a
hypothesis of error at most some value $\beta < 1/2$ (i.e., a bias
$\gamma=1/2-\beta>0$).  Most boosting
algorithms are {\em weight-based}, meaning they assign weights to each
example $x$ based solely on the performance of the hypotheses generated
so far on $x$, with probabilities proportional to
weights.\footnote{E.g.,
\cite{Schapire90,Freund90,FreundS97}.  (For Adaboost, we are
considering the version that uses a fixed upper bound $\beta$ on the
error of the weak hypotheses.)  Normalization may of course be
based on overall performance.}
We show here that any weight-based boosting algorithm can be applied
to achieve strong learning of any class with low overall
communication.  The key idea is that in each round, players need only
send enough data to the center for it to produce a weak hypothesis.
Once the weak hypothesis is constructed and broadcast to all the players,
the players can use it to
separately re-weight their own distributions and send data for the
next round.  No matter how large or small the weights become, each
round only needs a small amount of data to be transmitted.  Formally,
we show the following:
\begin{lemma}
\label{claim:boosting} Given any weight-based boosting algorithm
that achieves error $\epsilon$ by making $r(\epsilon,\beta)$ calls to
a $\beta$-weak learning oracle for $\HH$, we can construct a
distributed learning algorithm achieving error $\epsilon$ that uses
$O(r(\epsilon,\beta))$ rounds, each involving
$O((d/\beta)\log(1/\beta))$ examples and an additional $O(k
\log(d/\beta))$ bits of communication per round.
\end{lemma}

\begin{proof}
The key property of weight-based boosting algorithms that we will use
is that they maintain a current distribution such that
the probability mass on any example $x$ is solely a function of the
performance of the weak-hypotheses seen so far on $x$, except for a
normalization term that can be communicated efficiently.  This will
allow us to perform boosting in a distributed fashion.
Specifically, we run the boosting algorithm in rounds, as follows.
\begin{description}
\item[Initialization:] Each player $i$ will have a weight $w_{i,t}$
for round $t$.  We begin with $w_{i,0}=1$ for all $i$.  Let $W_t =
\sum_{i=1}^k w_{i,t}$ so initially $W_0 = k$.  These weights will all
be known to the center.  Each player $i$ will also have a large weighted
sample $S_i$, drawn from $D_i$, known only to itself. $S_i$ will be
weighted according to the specific boosting algorithm (and for all
standard boosting algorithms, the points in $S_i$ begin
with equal weights).  We now repeat the
following three steps for $t=1,2,3,\ldots$.
\item[1. Pre-sampling] The center determines the number of samples
$n_{i,t}$ to request from each player $i$ by sampling
$O(\frac{d}{\beta}\log \frac{1}{\beta})$ times from the multinomial
distribution
$w_{i,t-1}/W_{t-1}$.  It then sends each player $i$ the number
$n_{i,t}$, which requires only $O(\log \frac{d}{\beta})$ bits.
\item[2. Sampling] Each player $i$ samples $n_{i,t}$
examples from its local sample $S_i$ in proportion to its own internal
example weights, and sends them to the center.
\item[3. Weak-learning]   The center takes the union
of the received examples and uses these
$O(\frac{d}{\beta}\log\frac{1}{\beta})$
samples to produce a weak hypothesis $h_t$ of error at most $\beta/2$
over the current weighted distribution, which it then sends
to the players.\footnote{In fact, because we have a broadcast model,
technically the players each can observe all examples sent in step (2)
and so can simulate the center in this step.}
\item[4. Updating] Each player $i$, given $h_t$, computes the
new weight of each example in $S_i$ using the underlying boosting
algorithm and sends their sum $w_{i,t}$ to the center.  This sum can
be sent to sufficient accuracy using $O(\log \frac{1}{\beta})$ bits.
\end{description}
In each round, steps (1) and (2) ensure that the center receives
$O((d/\beta)\log(1/\beta))$ examples distributed according to
a distribution $D'$ matching that given by the boosting algorithm,
except for small rounding error due to the number of bits sent in step
(4).  Specifically, the variation distance between $D'$ and the
distribution given by the boosting algorithm is at most $\beta/2$.  Therefore,
in step (3), it computes a hypothesis $h_t$ with error at most
$\beta/2 + \beta/2 = \beta$
with respect to the current distribution given by the boosting
algorithm.  In step (4), the examples in all
sets $S_i$ then have their weights updated as determined by the
boosting algorithm, and the values $w_{i,t}$ transmitted ensure that
the normalizations are correct.
Therefore, we are simulating the underlying boosting algorithm having access to
a $\beta$-weak learner, and so the number of rounds is
$r(\epsilon,\beta)$.  The overall communication per round is
$O((d/\beta)\log(1/\beta))$ examples plus $O(k
\log(d/\beta))$ bits for communicating the numbers $n_{i,t}$ and
$w_{i,t}$, as desired.
\end{proof}

By adjusting the parameter $\beta$, we can trade off between
the number of rounds and communication complexity.  In particular,
using Adaboost \citep{FreundS97} in Lemma \ref{claim:boosting} yields
the following result (plugging in $\beta = 1/4$ or $\beta =
\epsilon^{1/c}$ respectively):
\begin{theorem}
\label{thm:boosting}
Any class $\HH$ can be learned to error $\epsilon$ in $O(\log
\frac{1}{\epsilon})$ rounds and $O(d)$ examples plus $O(k \log d)$ bits of
communication per round.  For any $c\geq 1$, $\HH$ can be
learned to error $\epsilon$ in $O(c)$ rounds and
$O(\frac{d}{\epsilon^{1/c}}\log \frac{1}{\epsilon})$ examples plus
$O(k \log \frac{d}{\epsilon})$ bits communicated per round.
\end{theorem}
Thus, any class of VC-dimension $d$ can be
learned using  $O(\log
\frac{1}{\epsilon})$ rounds and a total of $O(d \log
\frac{1}{\epsilon})$ examples, plus a small number of extra bits of
communication.

\section{Agnostic Learning}
\cite{BH12} show that any class $\HH$ can be agnostically learned to
error $O(\opt(\HH)) + \epsilon$ using only $\tilde{O}(d \log
1/\epsilon)$ label requests,
in an active learning model where
class-conditional queries are allowed.  We can use the core of their
result to agnostically learn any finite class $\HH$ to error
$O(\opt(\HH))+\epsilon$ in our setting, with a total communication that
depends only (poly)logarithmically on $1/\epsilon$.  The key idea is
that we can simulate their robust generalized halving algorithm using
communication proportional only to the number of class-conditional
queries their algorithm makes.

\begin{theorem}
\label{agnostic-BH} Any finite class $\HH$ can be learned to error
$O(\opt(\HH))+\epsilon$ with a total communication of
$O(k\log(|\HH|)\log\log(|\HH|)\log(1/\epsilon))$ examples and
$O(k\log(|\HH|)\log\log(|\HH|)\log^2(1/\epsilon))$ additional bits.
The latter may be eliminated if shared
randomness is available.
\end{theorem}

\begin{proof}
We prove this result by simulating the robust generalized halving
algorithm of \cite{BH12}, for the case of finite hypothesis spaces, in a
communication-efficient manner.\footnote{The algorithm of \cite{BH12}
for the case of infinite hypothesis spaces begins by using a large
unlabeled sample to determine a small $\epsilon$-cover of $\HH$.  This
appears to be difficult to simulate communication-efficiently.}
In particular, the algorithm operates as follows.
For this procedure, $N = O(\log\log |\HH|)$ and $s =
O(1/(\opt(\HH)+\epsilon))$ is such that the probability that the
best hypothesis in $\HH$ will have some error on a set of $s$ examples is a small constant..
\begin{enumerate}
\item We begin by drawing $N$ sets $S_1, \ldots, S_N$ of size $s$ from
$D$.  This can be implemented communication-efficiently as follows.
For $j=1,\ldots, N$, player 1 makes $s$ draws from $\{1,\ldots,k\}$ to
determine the number $n_{ij}$ of points in $S_j$ that should come from
each $D_i$.  Player 1 then sends each player $i$ the list $(n_{i1},
n_{i2}, \ldots, n_{iN})$, who draws (but keeps internally and does not
send) $n_{ij}$ examples of $S_j$ for each $1\leq j \leq N$.  Total
communication:
$O(kN\log(s))$ bits.   Note that if shared
randomness is available, then the computation of $n_{ij}$ can be
simulated by each player and so in that case no communication is
needed in this step.
\item Next we determine which sets $S_j$ contain an example
on which the majority-vote
hypothesis over $\HH$,  $maj(\HH)$,   makes a mistake, and identify one
such example $(\tilde{x}_j,\tilde{y}_j)$ for each such set.  We can
implement this communication-efficiently by having each player $i$
evaluate $maj(\HH)$ on its own portion of each set $S_j$ and broadcast
a mistake for each set on which at least one mistake is made.  Total
communication: $O(kN)$ examples.
\item If no more than $N/3$ sets $S_j$ contained a mistake for
$maj(\HH)$ then halt.  Else, remove from $\HH$ each $h$ that made
mistakes on more than $N/9$ of the identified examples
$(\tilde{x}_j,\tilde{y}_j)$, and go to (1). This step can be
implemented separately by each player without any communication.
\end{enumerate}
\cite{BH12} show that with high
probability the above process halts within $O(\log |\HH|)$ rounds,
does not remove the optimal $h \in \HH$, and
furthermore that when it halts, $maj(\HH)$ has error
$O(\opt(\HH))+\epsilon$.  The total amount of communication is
therefore $O(k\log(|\HH|)\log\log(|\HH|))$ examples and $O(k\log(|\HH|)\log\log(|\HH|)\log(1/\epsilon))$ additional bits.  The above has
been assuming that the value of $\opt(\HH)$ is known; if not then one
can perform binary search, multiplying the above quantities by an
additional $O(\log(1/\epsilon))$ term.  Thus, we achieve
the desired error rate within the desired communication bounds.
\end{proof}

\section{Privacy}
\label{sec:privacy}

In the context of distributed learning, it is also natural to consider
the question of privacy.  We begin by considering the well-studied
notion of differential privacy with respect to the examples, showing
how this can be achieved in many cases without any increase in
communication costs.  We then consider the case that one would like to
provide additional privacy guarantees for the players themselves.  One
option is to view each player as a single (large) example, but this
requires many players to achieve any nontrivial accuracy guarantees.
Thus, we also consider a natural notion of distributional privacy, in
which players do not view their distribution $D_i$ as sensitive, but
rather only the sample $S_i$ drawn from it.  We analyze how large a
sample is sufficient so that players can achieve accurate learning
while not revealing more information about their sample than is
inherent in the distribution it was drawn from.  We now examine each
notion in turn, and for each we explore how it can be achieved and the
effect on communication.

\subsection{Differential privacy with respect to individual examples}
\label{sec:diffpriv}
In this setting we imagine that each entity $i$ (e.g., a hospital) is
responsible
for the privacy of each example $x \in S_i$ (e.g., its patients).  In
particular, suppose $\sigma$ denotes a sequence of interactions between
entity $i$ and the other entities or center, and $\alpha>0$ is a given
privacy parameter.  Differential privacy asks that for any
$S_i$ and any modification $S_i'$ of $S_i$ in which any one example
has been arbitrarily changed, for all $\sigma$ we have $e^{-\alpha}
\leq \Pr_{S_i}(\sigma)/\Pr_{S_i'}(\sigma) \leq e^{\alpha}$, where
probabilities are over internal randomization of entity $i$. (See
\cite{Dwork06,Dwork08,Dwork09} for a discussion of motivations and
properties of differential privacy and a survey of results).

In our case, one natural approach for achieving privacy is to require
that all interaction with each entity $i$ be in the form of
statistical queries \citep{Kearns98}.  It is known that statistical
queries can be implemented in a privacy-preserving manner
\citep{DN04,BDMN05,KLNRS08}, and in particular that
a sample of size $O(\max[\frac{M}{\alpha\tau},
\frac{M}{\tau^2}]\log(M/\delta))$ is sufficient to preserve privacy
while answering $M$
statistical queries to tolerance $\tau$ with probability
$1-\delta$.  For completness, we present the proof below.

\begin{theorem}\citep{DN04,BDMN05,KLNRS08}
If $\HH$ is learnable using $M$ statistical queries of tolerance
$\tau$, then $\HH$ is learnable preserving differential privacy with
privacy parameter $\alpha$ from a sample $S$ of size
$O(\max[\frac{M}{\alpha\tau}, \frac{M}{\tau^2}]\log(M/\delta))$.
\end{theorem}
\begin{proof}
For a single statistical query, privacy with parameter $\alpha'$
can be achieved by adding Laplace noise of width
$O(\frac{1}{\alpha' |S|})$ to the empirical answer of the query on
$S$.  That is because changing a single entry in $S$ can change the
empirical answer by at most $1/|S|$, so by adding such noise we have
that for any $v$, $\Pr_{S}(v)/\Pr_{S'}(v) \leq e^{\alpha'}$.
Note that with
probability at least $1-\delta'$, the amount of noise added to any
given answer is at most $O(\frac{1}{\alpha' |S|}\log(1/\delta'))$.
Thus, if the overall algorithm requires $M$ queries to be answered to
tolerance $\tau$, then setting $\alpha'=\alpha/M, \delta'=\delta/(2M),
\tau = O(\frac{1}{\alpha'|S|} \log(1/\delta'))$,
privacy can be achieved so long as we have
$|S| = O(\max[\frac{M}{\alpha\tau},
\frac{M}{\tau^2}]\log(M/\delta))$, where the second term of the max is
the sample
size needed to achieve tolerance $\tau$ for $M$ queries even without privacy
considerations.
As described in \cite{DRV10}, one can achieve a somewhat weaker privacy guarantee
using $\alpha'=O(\alpha/\sqrt{M})$.
\end{proof}

However, this generic approach
may involve significant communication overhead over the best
non-private method.  Instead, in many cases we can achieve privacy
without any communication overhead at all by performing statistical
queries {\em internally to the entities}.  For example, in the case of
intersection-closed classes, we have the following privacy-preserving
version of Theorem
\ref{thm:iclosed}.
\begin{theorem}
\label{thm:privpos}
If $\HH$ can be properly learned via statistical queries to $D^+$
only, then $\HH$ can be learned using one round and $k$
hypotheses of total communication while preserving differential
privacy.
\end{theorem}
\begin{proof}
Each entity $i$ learns a hypothesis $h_i \in \HH$ using
privacy-preserving statistical queries to its own $D_i^+$, and sends
$h_i$ to the center.  Note that $h_i \subseteq f$ because the
statistical query algorithm must succeed for any possible $D^-$.
Therefore, the center can simply compute the minimal $h \in \HH$ such
that $h \supseteq h_i$ for all $i$, which will have error at most
$\epsilon$ over each $D_i$ and therefore error at most $\epsilon$ over
$D$.
\end{proof}
For instance, the class of conjunctions can be learned via statistical
queries to $D^+$ only by producing the conjunction of all variables
$x_j$ such that $\Pr_{D_i^+}[x_j=0] \leq \frac{\epsilon}{2n} \pm
\tau$, for $\tau = \frac{\epsilon}{2n}$.  Thus, Theorem
\ref{thm:privpos} implies that conjunctions can be learned in a
privacy-preserving manner without any communication overhead.

Indeed, in all the algorithms for specific classes given in this
paper, except for parity functions, the interaction between entities
and their data can be simulated with statistical queries.  For
example, the decision list algorithm of
Theorem \ref{thm:declist} can be implemented by having each entity
identify rules to send to the center via statistical queries to
$D_i$.  Thus, in these or any other cases where the information
required by the protocol can be extracted by each entity
using statistical queries to its own data, there is no communication
overhead due to preserving privacy.

\subsection{Differential privacy with respect to the entities}

One could also ask for a stronger privacy guarantee, that each
{\em entity} be able to plausibly claim to be holding any other
dataset it wishes; that is, to require $e^{-\alpha}
\leq \Pr_{S_i}(\sigma)/\Pr_{S'}(\sigma) \leq e^{\alpha}$ for all $S_i$
and all (even unrelated) $S'$.  This in fact corresponds precisely to
the {\em local privacy} notion of \cite{KLNRS08}, where in essence the
only privacy-preserving mechanisms possible are via
randomized-response.\footnote{For example, if an entity is asked a
question such as ``do you have an example with $x_i=1$'', then it
flips a coin and with probability $1/2 + \alpha'$ gives the correct
answer and with probability $1/2 - \alpha'$ gives the incorrect
answer, for some appropriate $\alpha'$.}  They show that any
statistical query algorithm can be implemented in such a setting;
however, because each entity is now viewed as essentially a
single datapoint, to achieve any nontrivial accuracy, $k$
must be quite large.

\subsection{Distributional privacy}
If the number of entities is small, but we still want
privacy with respect to the entities themselves, then one type
of privacy we {\em can} achieve is a notion of {\em distributional
privacy}.  Here we guarantee that that each player $i$
reveals (essentially) no more information about
its own sample $S_i$ than is inherent in $D_i$ itself.  That is, we
think of $S_i$ as ``sensitive'' but $D_i$ as ``non-sensitive''.
Specifically, let us say a probabilistic mechanism $A$ for answering a
request $q$ satisfies {\em $(\alpha,\delta)$ distributional privacy} if
$$\Pr_{S,S' \sim D_i}\left[\forall v, \;  e^{-\alpha} \leq \Pr_A(A(S,q) =
v)/ \Pr_A(A(S',q) = v) \leq e^{\alpha}\right] \geq 1-\delta.$$
In other words, with high probability, two random samples $S,S'$ from $D_i$
have nearly the same probability of producing any given answer
to request $q$.
\cite{BLR08} introduce a similar privacy notion,\footnote{In the
notion of \cite{BLR08}, $D_i$ is uniform over some domain and
sampling is done without replacement.} which they show is strictly
stronger than differential privacy, but do not provide efficient
algorithms.   Here, we show how distributional privacy  can be
implemented efficiently.

Notice that in this context, an ideal privacy preserving mechanism
would be for player $i$ to somehow use its sample to reconstruct $D_i$
perfectly and then draw a ``fake'' sample from $D_i$ to use
in its communication protocol.  However, since reconstructing $D_i$
perfectly is not in general possible, we instead will work via
statistical queries.

\begin{theorem}
\label{thm:distribpriv}
If $\HH$ is learnable using $M$ statistical queries of tolerance
$\tau$, then $\HH$ is learnable preserving distributional privacy from
a sample of size $O(\frac{M^2\log^3(M/\delta)}{\alpha^2\tau^2})$.
\end{theorem}

\begin{proof}
We will show that we can achieve distributional privacy using statistical
queries by adding additional Laplace noise beyond that required solely
for differential privacy of the form in Section \ref{sec:diffpriv}.

Specifically, for any statistical query $q$, Hoeffding bounds imply
that with probability at least $1-\delta'$, two random samples
of size $N$ will produce answers within $\beta = O(\sqrt{\log(1/\delta')/N})$
of each other (because each will be within $\beta/2$ of the
expectation with probability at least $1-\delta'/2$).  This quantity
$\beta$ can now be viewed as the ``global
sensitivity'' of query $q$ for distributional privacy.  In particular,
it suffices to add Laplace
noise of width $O(\beta/\alpha')$ in order to achieve privacy
parameter $\alpha'$ for this query $q$ because we have that with
probability at least $1-\delta'$, for two random samples $S,S'$ of
size $N$, for any $v$, $\Pr(A(S,q)=v)/Pr(A(S',q)=v) \leq
e^{\beta/(\beta/\alpha')} = e^{\alpha'}$.  Note that this
has the property that with probability at least $1-\delta'$, the
amount of noise added to any given answer is at most $O((\beta
/\alpha')\log(1/\delta'))$.

If we have a total of $M$ queries, then it suffices for preserving
privacy over the entire sequence to set $\alpha'= \alpha/M$ and
$\delta' = \delta/M$.  In order to have each query answered with high
probability to within $\pm \tau$, it suffices to have $\beta +
(\beta/\alpha')\log(1/\delta') \leq c\tau$ for some constant $c$,
where the additional (low-order) $\beta$ term is just the statistical
estimation error without added noise.  Solving for $N$, we find that a
sample of size $N = O(\frac{M^2\log^3(M/\delta)}{\alpha^2\tau^2})$ is
sufficient to maintain distributional privacy while answering each
query to tolerance $\tau$, as desired.
\end{proof}

As in the results of Section \ref{sec:diffpriv}, Theorem
\ref{thm:distribpriv} implies that if each player can
run its portion of a desired communication protocol while only
interacting with its own data via statistical queries, then so long
as $|S_i|$ is sufficiently large, we can
implement distributional privacy without any communication penalty by
performing internal statistical queries privately as above.
For example, combining Theorem \ref{thm:distribpriv} with the proof of Theorem \ref{thm:privpos} we have:
\begin{theorem}
\label{thm:distprivpos}
If $\HH$ can be properly learned via statistical queries to $D^+$
only, then $\HH$ can be learned using one round and $k$
hypotheses of total communication while preserving distributional
privacy.
\end{theorem}

\bibliography{bibfile}

\newpage 
\appendix

\section{Table of results}
\label{app:table}
\begin{center}
\begin{tabular}{|l|l|c|}\hline
{\bf Class / Category} & {\bf Communication} & {\bf Efficient?}\\\hline
Conjunctions over $\{0,1\}^n$ & $O(nk)$ bits & yes\\ \hline
Parity functions over $\{0,1\}^n$, $k=2$  & $O(n)$ bits & yes \\
\hline
Decision lists over $\{0,1\}^n$ & $O(nk \log n)$ bits & yes \\ \hline
Linear separators in $R^d$ & $O(d\log(1/\epsilon))$ examples$^*$& yes \\
~~~ under radially-symmetric $D$ & $O(k)$ examples & yes \\
~~~ under $\alpha$-well-spread $D$ & $O(k(1+\alpha/\gamma^2))$ hypotheses & yes \\
~~~ under non-concentrated $D$ & $O(k^2\sqrt{d
\log(dk/\epsilon)}/\epsilon^2)$ hyps & yes \\ \hline
General Intersection-Closed & $k$ hypotheses & see Note 1 below \\ \hline
Boosting & $O(d\log 1/\epsilon)$ examples$^*$& see Note 2 below \\ \hline
Agnostic learning & $\tilde{O}(k\log(|\HH|)\log(1/\epsilon))$ exs$^*$ & see Note 3 below \\ \hline
\end{tabular}
\end{center}


\noindent
$*$: plus low-order additional bits of communication. 

\noindent
{\bf Note 1:} Efficient if can compute the smallest consistent
hypothesis in $\HH$ efficiently, and for any given $h_1, \ldots,
h_k$, can efficiently compute the minimum $h \supseteq h_i$ for
all $i$.

\noindent
{\bf Note 2:} Efficient if can efficiently weak-learn with $O(d)$ examples.

\noindent
{\bf Note 3:} Efficient if can efficiently run robust halving algorithm for $\HH$. 
\section{Additional simple cases}

\subsection{Distribution-based algorithms}
\label{app:interval}

An alternative basic
approach, in settings where it can be done succinctly, is for each
entity $i$ to send to the center a representation of its (approximate)
distribution over labeled data.  Then, given the descriptions, the center
can deduce an approximation of the overall distribution over labeled
data and search for a near optimal hypothesis. This example is
especially relevant for the
agnostic $1$-dimensional case, e.g., a union of $d$ intervals over
$X=[0,1]$. Each entity first simply
sorts the points, and determines $d/\epsilon$
border points defining regions of probability mass (approximately)
$\epsilon/d$. For each segment between two border points, the entity reports
the fraction of positive versus negative examples.  It additionally
sends the border points themselves.  This communication
requires $O(d/\epsilon)$ border points and an additional $O(\log
d/\epsilon)$ bits to report the fractions within each such
interval, per entity. Given this information, the center can
approximate the best union of $d$ intervals with error
$O(\epsilon)$. Note that the supervised
learning baseline algorithm  would have a bound of
$\tilde{O}(d/\epsilon^2)$ in terms of the number of points
communicated.

\begin{theorem}
There is an algorithm for agnostically learning a union of $d$
intervals that uses one round and $O(kd/\epsilon)$ values
(each either a datapoint or a $\log d/\epsilon$ bit integer), such
that the final hypothesis produced has error $\opt(\HH) + \epsilon$.
\end{theorem}

\subsection{Version space algorithms}
\label{app:ver}

Another simple case where one can perform well is when the version
space can be compactly described.  The version space of $\HH$ given
a sample $S_i$ is the set of all $h\in \HH$ which are consistent
with $S_i$. Denote this set by $VerSp(\HH,S_i)$.

\noindent{\em Generic Version Space Algorithm:} Each entity sends
$VerSp(\HH,S_i)$ to the center. The center computes $V=\cap_i
VerSp(\HH,S_i)$. Note that $V=VerSp(\HH,\cup_i S_i)$. The center can
send either $V$ or some $h\in V$.

\smallskip
\noindent
{\em Example (linear separators in $[0,1]^2$):} Assume that the points
have margin $\gamma$. We can cover a convex set in $[0,1]^2$ using
$1/\gamma^2$ rectangles, whose union completely covers the convex
set, and is completely covered by the convex set extended by $\gamma$.
Each entity does this for its positive and negative regions, sending
this (approximate) version space to the center.
This gives a one-round algorithm
with communication cost of $O(1/\gamma^2)$ points.

\section{Linear Separators: Margin lower bound}
\label{app:per_margin}

\begin{proof}(Theorem~\ref{claim:per_margin})
Suppose we have two players, each with their own set of examples, such
that the combined dataset has a linear separator of margin $\gamma$.
Suppose furthermore we run the perceptron algorithm where each player
performs updates on their own dataset until consistent (or at least
until low-error) and then passes the
hypothesis on to the other player, with the process continuing until one
player receives a hypothesis that is already low-error on its own
data.  How many rounds can this take in the worst case?

Below is an example showing a problematic case where this can indeed
result in $\Omega(1/\gamma^2)$ rounds.

In this example, there are 3 dimensions and the target vector is $(0,1,0)$.
Player 1 has the positive examples, with 49\% of its data points at location
$(1,\gamma,3\gamma)$ and 49\% of its
data points are at location $(1,\gamma,-\gamma)$.  The remainder of
player 1's points are at location $(1,\gamma,\gamma)$.
Player 2 has the negative examples.  Half of its data points are at
location $(1,-\gamma,-3\gamma)$ and half of its data points are at location
$(1,-\gamma,\gamma)$.

The following demonstrates a bad sequence of events that can occur,
with the two players essentially fighting over the first
coordinate:
\begin{center}
\begin{tabular}{r|l|l}
player & updates using & producing hypothesis \\ \hline
player 1 & $(1,\gamma,\gamma)$, $+$   & $(1,\gamma,\gamma)$ \\
player 2 & $(1,-\gamma,-3\gamma)$, $-$ & $(0,2\gamma,4\gamma)$ \\
player 2 & $(1,-\gamma,\gamma)$, $-$  & $(-1,3\gamma,3\gamma)$ \\
player 1 & $(1,\gamma,3\gamma)$, $+$  & $(0,4\gamma,6\gamma)$ \\
player 1 & $(1,\gamma,-\gamma)$, $+$  & $(1,5\gamma,5\gamma)$ \\
player 2 & $(1,-\gamma,-3\gamma)$, $-$  & $(0,6\gamma,8\gamma)$ \\
player 2 & $(1,-\gamma,\gamma)$, $-$  & $(-1,7\gamma,7\gamma)$ \\
player 1 & $(1,\gamma,3\gamma)$, $+$  & $(0,8\gamma,10\gamma)$ \\
player 1 & $(1,\gamma,-\gamma)$, $+$  & $(1,9\gamma,9\gamma)$ \\
... & &
\end{tabular}
\end{center}
Notice that when the hypothesis looks like $(-1,k\gamma,k\gamma)$, then
the dot-product with the example $(1,\gamma,3\gamma)$ from player 1 is
$-1 + 4k\gamma^2$.  So long as this is negative, player 1 will make
two updates producing hypothesis $(1,(k+2)\gamma,(k+2)\gamma)$.  Then,
so long as $4(k+2)\gamma^2 < 1$, player 2 will make two updates
producing hypothesis $(-1, (k+4)\gamma, (k+4)\gamma)$.  Thus, this
procedure will continue for $\Omega(1/\gamma^2)$ rounds.
\end{proof}

\end{document}